\documentclass[]{article}
\usepackage{subfigure}
\usepackage[colorinlistoftodos]{todonotes}
\usepackage{comment}
\usepackage[section]{placeins}

\usepackage{hyperref}
\hypersetup{plainpages=false,
 	breaklinks=true,
     colorlinks=true,
     urlcolor=black,
     citecolor=black,                                       
     linkcolor=black,
     bookmarksopen=true,
     bookmarksdepth=3,
     bookmarksopenlevel=3,
     pdfstartview=FitV,
     pdfdisplaydoctitle=true}

\usepackage[american]{babel}
\usepackage[backend=biber,natbib=true,style=apa,citestyle=authoryear,giveninits=true,uniquename=false,useprefix=true,apamaxprtauth=99,maxbibnames=99,sorting=nyt]{biblatex}
\DeclareLanguageMapping{american}{american-apa}
\defbibheading{bibliography}[\refname]{}

\usepackage{algorithm}
\usepackage{algorithmic}

\usepackage[utf8]{inputenc}
\usepackage{csquotes}
\usepackage{amsmath}
\usepackage{amsthm}
\usepackage{amsfonts}
\usepackage{comment}
\usepackage{mdframed}
\usepackage{xstring}
\usepackage{color, colortbl}
\usepackage{sidecap}

\usepackage[letterpaper, margin=1in]{geometry}

\newtheorem{theorem}{Theorem}

\newtheorem{lemma}{Lemma}
\newtheorem{assumption}{Assumption}

\newcommand{\states}{\mathcal{S}}
\newcommand{\actions}{\mathcal{A}}

\newcommand{\expect}{\mathbb{E}}

\newcommand{\td}{TD($\lambda$)~}
\newcommand{\varg}{v}
\newcommand{\varghat}{V}
\newcommand{\val}{j}
\newcommand{\valhat}{J}

\newcommand{\Exp}[1]{\expect\left[ #1 \right]}

\newcommand{\second}{M}
\newcommand{\lambdaone}{\kappa}
\newcommand{\lambdatwo}{\lambda}
\newcommand{\lambdathree}{\bar{\kappa}}
\newcommand{\secondrho}{\eta}
\newcommand{\trace}{E}

\newcommand{\C}{R}
\renewcommand{\c}{r}

\newcommand{\valtext}{$J$}
\newcommand{\vartext}{$V$}

\newcommand{\metareward}{meta-reward}
\newcommand{\direct}{direct}

\makeatletter
\renewcommand\@date{{%
  \vspace{-\baselineskip}%
  \large\centering
  \begin{tabular}{@{}c@{}}
    Craig Sherstan
    \\ \normalsize sherstan@ualberta.ca
  \end{tabular}%
  \quad
  \begin{tabular}{@{}c@{}}
    Brendan Bennett
    \\    \normalsize babennet@ualberta.ca
  \end{tabular}
  \quad
  \begin{tabular}{@{}c@{}}
    Kenny Young
    \\ \normalsize kjyoung@ualberta.ca
  \end{tabular}
  
  \bigskip
  \begin{tabular}{@{}c@{}}
    Dylan R. Ashley
    \\    \normalsize dashley@ualberta.ca
  \end{tabular}%
  \quad
  \begin{tabular}{@{}c@{}}
    Adam White
    \\    \normalsize amw8@ualberta.ca
  \end{tabular}
  \quad
  \begin{tabular}{@{}c@{}}
    Martha White
    \\ \normalsize whitem@ualberta.ca
  \end{tabular}
  
  \bigskip
  \begin{tabular}{@{}c@{}}
  	Richard S. Sutton 
    \\   	\normalsize rsutton@ualberta.ca
  \end{tabular}
  \bigskip
  
University of Alberta, Department of Computing Science

}}

\begin{document}

%

%


\title{Directly Estimating the Variance of the $\lambda$-Return Using Temporal-Difference Methods}

\maketitle



\begin{abstract}
This paper investigates estimating the variance of a temporal-difference learning agent's update target.
Most reinforcement learning methods use an estimate of the value function, which captures how good it is for the agent to be in a particular state and is mathematically
expressed as the expected sum of discounted future rewards (called the return). These values
can be straightforwardly estimated by averaging batches of returns using Monte Carlo methods. However, if we wish to update the agent's value estimates during learning--before terminal outcomes are observed--we must use a different estimation target called the $\lambda$-return, which truncates the return with the agent's own estimate of the value function. Temporal difference learning methods estimate the expected $\lambda$-return for each state, allowing these methods to update online and incrementally, and in most cases achieve better generalization error and faster learning than Monte Carlo methods. Naturally one could attempt to estimate higher-order moments of the $\lambda$-return. This paper is about estimating the variance of the $\lambda$-return. Prior work has shown that given estimates of the variance of the $\lambda$-return, learning systems can be constructed to (1) mitigate risk in action selection, and (2) automatically adapt the parameters of the learning process itself to improve performance. Unfortunately, existing methods for estimating the variance of the $\lambda$-return are complex and not well understood empirically. We contribute a method for estimating the variance of the $\lambda$-return directly using policy evaluation methods from reinforcement learning. Our approach is significantly simpler than prior methods that independently estimate the second moment of the $\lambda$-return. Empirically our new approach behaves at least as well as
existing approaches, but is generally more robust.
\end{abstract}

\section{Introduction}

\begin{figure}

	\includegraphics[width=\linewidth]{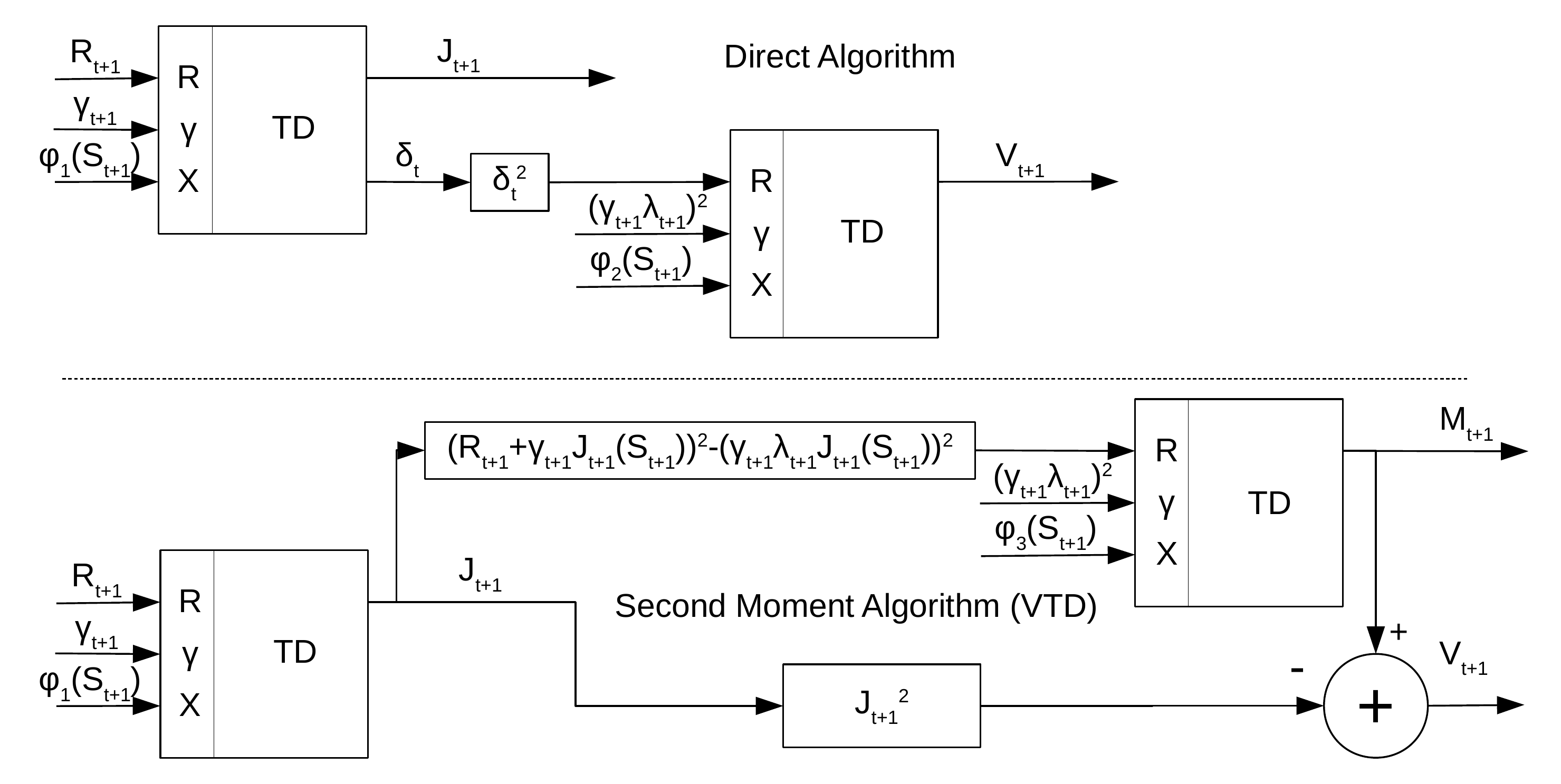}
    	\caption{Each TD node takes as input a reward $\C$, a discounting function $\gamma$, and features $\phi$. 
 For the direct method (\textbf{top}) the squared TD error of the first-stage value estimator is used as the \metareward{} for the second-stage \vartext{} estimator. For VTD (\textbf{bottom}), a more complex computation is used for the \metareward{} and an extra stage of computation is required.       
      }
    \label{fig:networks}
\end{figure}




Conventionally in reinforcement learning, the agent estimates the expected value of the return---the discounted sum of future rewards, as an intermediate step to find an optimal policy. Given a trajectory of experience, the agent can average the returns observed from each state. To estimate the value function online---while the trajectory unfolds---we update the agent's value estimates towards the expected $\lambda$-return. The $\lambda$-return has the same expected value as the return, but can be estimated online using a memory trace. Algorithms that estimate the expected value of the $\lambda$-return are called temporal-difference learning methods.
The first moment, however, is not the only statistic that can be estimated. 
In addition to the expected value, we could estimate the variance of the $\lambda$-return. 

An estimate of the variance of the $\lambda$-return can be used in several ways to improve estimation and decision-making. 
\citet{Sato2002,Ghavamzadeh,Tamar2012,Tamar2013b} use an estimate of the variance of the $\lambda$-return to design algorithms that account for risk in decision making. Specifically they formulate the agent's objective as maximizing reward, while minimizing the variance of the  $\lambda$-return.
\citet{White2016b} estimated the variance of the $\lambda$-return, \vartext{}, to automatically adapt the trace-decay parameter, $\lambda$, used in learning updates. This resulted in faster learning for the agent, but more importantly removed the need to tune $\lambda$ by hand.

The variance \vartext{} can be estimated directly or indirectly. Indirect estimation involves estimating the first moment (the value $\valhat$) and second moment ($M$) of the return and taking their difference as: $\varghat(s)=M(s)-\valhat(s)^2$. \citet{Sobel1982} were the first to formulate a Bellman operators for $M$. Later \citet{Tamar2016,Tamar2013b,Ghavamzadeh}, extended \citet{Sobel1982}'s approach to estimating the variance for $\lambda = 0$ to $\lambda =1$. Finally, \citet{White2016b} introduced an estimation method called VTD, that supports off-policy learning \citep{Sutton2009,Maei2011}, state-dependent discounts and state-dependent trace-decay parameters. 
An alternative approach is to estimate the variance of the $\lambda$-return \vartext{} directly. This has been considered by \citet{Tamar2012}, but they were unable to derive a Bellman operator---instead giving a Bellman-like operator---and considered only cost-to-go problems. 

In this paper, we show that one can use temporal-difference learning, a online method for estimating value functions \citep{Sutton1988}, to estimate \vartext{} directly. Our new method supports off-policy learning, state-dependent discounts, and state-dependent trace-decay parameters. 
We introduce a new Bellman operator for the variance of the $\lambda$-return, and further prove that even for a value function that does not satisfy the Bellman operator for the expected $\lambda$-return, the error in this recursive formulation is proportional to the error in the value function approximation. Interestingly, the Bellman operator for the second moment requires an unbiased estimate of the $\lambda$-return \citep{White2016b}; our Bellman operator for the variance avoids this term, and so has a simpler update. 
Both our direct method and VTD can be seen as a network of two TD estimators running sequentially (Figure~\ref{fig:networks}). 
 


Our goal is to understand the empirical properties of the direct and indirect approaches for estimating variance, 
as neither have yet been thoroughly studied. 
%
In general, we found that direct estimation is just as good as VTD, and in many cases better. Specifically, we observe that the direct approach is better behaved in the early stages of learning before the value function has converged. Further, we observe that the variance of the \vartext{} estimates can be higher for VTD under several circumstances: (1) when there is a mismatch in step-size between the value estimator and the \vartext{} estimator, (2) when traces are used with the value estimator, (3) when estimating \vartext{} of the off-policy return, and (4) when there is error in the value estimate.
Overall, we conclude that the direct approach to estimating \vartext{} is both simpler and better behaved than VTD.

\section{The MDP Setting}


We model the agent's interaction with the environment as a finite Markov decision process (MDP) consisting of a finite set of states $\states$, a finite set of actions, $\actions$, and a transition model $p: \states \times \states \times \actions \rightarrow [0,1]$ defining the probability $p(s'|s, a)$ of transition from state $s$ to $s'$ when taking action $a$. In the policy evaluation setting considered in this paper,
the agent follows a fixed policy $\pi(a|s)\in [0,1]$ that provides the probability of taking action $a$ in state $s$.
At each timestep the agent receives a random reward $R_{t+1}$, dependent only on $S_t, A_t, S_{t+1}$.


The return is the discounted sum of future rewards 
%
\begin{equation}
\begin{aligned}
\label{eq:MCreturn}
G_{t}&=\C_{t+1} + \gamma_{t+1}\C_{t+2} + \gamma_{t+1}\gamma_{t+2}\C_{t+3} + \ldots \\
&=\C_{t+1} + \gamma_{t+1}G_{t+1}.
\end{aligned}
\end{equation}
The discount function $\gamma: \states \rightarrow [0,1]$, with $\gamma_{t}\equiv\gamma(S_t)$, provides a variable level
of discounting depending on the state \citep{Sutton2011}.  
%
The value of a state, $\val(s)$, is defined as the expected return from state $s$ under a particular policy $\pi$
%
%
\begin{align}
\val(s)=&\expect_{\pi}[G_t|S_t=s].
\label{eq:value}
\end{align}
We use $\val$ to indicate the true value function and $\valhat$ the estimate.
The TD-error is the difference between the one-step approximation and the current estimate:
\begin{align}
\delta_t=\C_{t+1}+\gamma_{t+1}\valhat_t(S_{t+1})-\valhat_t(S_t).
\label{eq:TDerr}
\end{align}
The $\lambda$-return 
%
\begin{equation*}
G_t^{\lambda}=\C_{t+1}+\gamma_{t+1}(1-\lambda_{t+1})\valhat_t(S_{t+1}) +\gamma_{t+1}\lambda_{t+1} G^{\lambda}_{t+1}
\end{equation*}
provides a bias-variance trade-off by incorporating $\valhat$, which is a potentially lower-variance but biased estimate of the return.
This trade-off is determined by a state-dependent trace-decay parameter, $\lambda_t\equiv\lambda(S_t) \in [0,1]$.
When $\valhat_t(S_{t+1})$ is equal to the expected return from $S_{t+1}=s$, then
$\expect_{\pi}[(1-\lambda_{t+1})\valhat_t(S_{t+1})
 +\gamma_{t+1}\lambda_{t+1} G^{\lambda}_{t+1}|S_{t+1}=s] = \expect_{\pi}[G^{\lambda}_{t+1}|S_{t+1}=s]$, and so the $\lambda$-return is unbiased. Beneficially, however, the expected value $\valhat_t(S_{t+1})$ is lower-variance than the sample $G^{\lambda}_{t+1}$. 
If $\valhat_t$ is inaccurate, however, some bias is introduced. 
Therefore, when $\lambda=0$, the $\lambda$-return is lower-variance but can be biased. When $\lambda=1$, the $\lambda$-return equals the Monte Carlo return (Equation~\eqref{eq:MCreturn}); in this case, the update target exhibits more variance, but no bias. 
In the tabular setting evaluated in this paper, $\lambda$ does not affect the fixed point solution of the value estimate, only the rate at which learning occurs. It does, however, affect the observed variance of the return, which we estimate. 
The $\lambda$-return is implemented using traces as in the following \td algorithm, shown with accumulating traces:
\begin{equation}
\begin{aligned}
	\trace_t(s)&\leftarrow\begin{cases}
      \gamma_t\lambda_t \trace_{t-1}(s) + 1 & s=S_t \\
      \gamma_t\lambda_t \trace_{t-1}(s) & \forall s \in \states, s \ne S_t
   \end{cases}\\
\valhat_{t+1}(S_t)&\leftarrow \valhat_t(S_t) + \alpha\delta_t \trace_t(S_t)
\end{aligned}
\label{eq:td_replacing_traces}
\end{equation}

\section{Estimating the Variance of the Return}
\label{sec:derivation}



When estimating \vartext, we have both a value estimator
and a variance estimator. The \textit{value estimator} provides an estimate
of the expected $\lambda$-return, known as the policy evaluation problem.
The \textit{variance estimator} provides an estimate of the variance of the $\lambda$-return. We show below how we can similarly use any TD method to learn the variance estimator, such as TD with accumulating traces (Equation~\ref{eq:td_replacing_traces}).

Because we have two separate TD estimators---one each for \valtext{} and \vartext{} --- they can select different trace-decay parameters for learning. In fact, as done by \citet{White2016b}, the value estimator can use a different trace-decay parameter than the $\lambda$-return for which we are estimating \vartext. This is because the $\lambda$-return is defined for any given value function, regardless of how that value function is estimated.  
There are three possible trace-decay parameters: 
1) the $\lambda$ of the $\lambda$-return for which variance is being estimated, 2) that used by the traces of the value estimator ($\lambdaone$), 3) that used by the traces of the variance estimator ($\lambdathree$). 

We summarize the notation here for easy reference. 
Variables without the bar refer to the value estimator and variables with bars refer to the variance estimator. 
\begingroup
\allowdisplaybreaks
\begin{align*}
\val-&\text{true value function of the target policy $\pi$.}\\
\valhat-&\text{estimate of $\val$.}\\
\C-&\text{reward used in the value function estimate.}\\
\bar{\C}-&\text{\metareward{} used in the variance estimate.}\\
\lambdatwo -&\text{bias-variance parameter of the target $\lambda$-return.}\\
\lambdaone -&\text{trace-decay parameter of the value estimator.}\\
\lambdathree -&\text{trace-decay parameter of the secondary estimator.}\\
\gamma -&\text{discounting function used by the value estimator.}\\
\bar{\gamma} -&\text{discounting function used by the variance estimator.}\\
\delta_t -&\text{TD error of the value function at time $t$.}\\
\bar{\delta_t} -&\text{TD error of the variance estimator at time $t$.}\\
\second-&\text{estimate of the second moment.}\\
\varg-&\text{true variance of the return.}\\
\varghat-&\text{estimate of $\varg$.}
\end{align*}
\endgroup
Our direct algorithm, shown here, uses TD(0) to estimate variance. For an expanded implementation with traces in the off-policy setting see Appendix~\ref{appendix:general}.
\\\\
\textbf{Direct Variance Algorithm}
\begin{equation}
\begin{aligned}
\bar{\gamma}_{t+1}& \leftarrow \gamma_{t+1}^2 \lambdatwo_{t+1}^2\\
\bar{\C}_{t+1}& \leftarrow \delta_t^2\\
\bar{\delta}_t& \leftarrow \bar{\C}_{t+1} + \bar{\gamma}_{t+1}\varghat_t(s') - \varghat_t(s)\\
\varghat_{t+1}(s)&\leftarrow \varghat_t(s)+\bar{\alpha} \bar{\delta}_t
\end{aligned}
\label{alg:direct}
\end{equation}
An alternative to this direct method is to instead estimate the second moment. 
The variant shown here is equivalent to on-policy VTD with no traces, $\lambdathree=0$, and the step-size for the second set of weights set to 0. Further, the Tamar TD(0) algorithm (\cite{Tamar2016}) can be recovered from Equation \ref{alg:comparison} by using $\lambdaone=0, \lambdatwo=1, \lambdathree=0$. This algorithm does not impose that the variance be non-negative.
\\\\
\textbf{Second Moment Algorithm (VTD)}
\begin{align}
\bar{\gamma}_{t+1}\leftarrow&\gamma_{t+1}^2 \lambdatwo_{t+1}^2 \nonumber\\
\bar{\C}_{t+1}\leftarrow&(\C_{t+1}+\gamma_{t+1}\valhat_{t+1}(s'))^2- \bar{\gamma}_{t+1}\valhat_{t+1}(s')^2 \nonumber\\
\bar{\delta}_t\leftarrow&\bar{\C}_{t+1} + \bar{\gamma}_{t+1}\second_t(s') - \second_t(s) \label{alg:comparison}
\\
\second_{t+1}(s)\leftarrow&\second_t(s)+\bar{\alpha} \bar{\delta}_t \nonumber\\
\varghat_{t+1}(s) =& \second_{t+1}(s)-\valhat_{t+1}(s)^2 \nonumber
\end{align}

\section{Derivation of the Direct Method}

The derivation of the direct method follows from characterizing the Bellman operator for the variance of the $\lambda$-return.
%
Theorem \ref{recurse_var} gives a Bellman equation for the variance $\varg$. 
It has precisely the form of a TD target with \metareward{} $\bar{\delta}_t=\delta_t^2$ and discounting function $\bar{\gamma}_{t+1}=\gamma_{t+1}^2\lambda_{t+1}^2$. 
Therefore, we can conveniently estimate \vartext{} using TD methods.
Further, we show that even when the value function does not satisfy the Bellman equation,
this results only in a proportional error in the variance estimator. 
We first show the result for the on-policy setting, for simplicity; 
the more general off-policy algorithm is provided in Appendix \ref{appendix:general}

This result provides the first general Bellman operator directly for the variance. 
The Bellman operators for the variance are general, in that they allow for either the episodic or continuing setting,
by using variable $\gamma$.
Interestingly, by directly estimating variance, we avoid a second term in the cumulant,
that is present in approaches that estimate the second moment \citep{Tamar2013b,Tamar2016,White2016b}.
While \citet{Tamar2012} also developed an approach to directly estimate the variance, their method defined a non-linear Bellman operator and is restricted to cost-to-go problems. 
Follow-up work moved to estimating the second-moment instead \citep{Tamar2013b,Tamar2016}, but with simplifying assumptions that only considered expected reward from a state and assuming $\lambda = 1$. 
The work developing VTD generalizes to any $\lambda$, but does not characterize error when using an inaccurate value function. 





To have a well-defined solution to the fixed point, we need the discount to be less than one
for some transition \citep{White2017,Yu2015}. This corresponds to assuming that the policy is proper,
for the cost-to-go setting \citep{Tamar2016}.
\begin{assumption}
The policy reaches a state $s$ where $\gamma(s) < 1$ in a finite number of steps. 
\end{assumption}

\begin{theorem}
For any $s \in \states$, 
\label{recurse_var}
\begin{align}
\val(s)&=\expect\Big[ R_{t+1} + \gamma_{t+1} \valhat(S_{t+1})\ | \ S_t=s\Big] \nonumber\\
\varg(s) 
&=\expect\Big[\delta_t^2+\gamma_{t+1}^2\lambda_{t+1}^2\varg(S_{t+1})\ | \ S_t=s\Big] \label{eq_bellman_var}
\end{align}
Further, for approximate value function $\valhat$, if there is an $\epsilon: \states \rightarrow [0,\infty)$ bounding
value estimates $(\valhat(s) - \val(s))^2 \le \epsilon(s)$
and covariance terms $|\Exp{ \gamma_{t+1} \lambda_{t+1} \delta_{t} (\val(S_{t+1}) - \valhat(S_{t+1})) |S_t = s} | \le \epsilon(s)$, then
%
\begin{align*}
\left|\varghat(s) - \expect\Big[\delta_t^2+\gamma_{t+1}^2\lambda_{t+1}^2\varghat(S_{t+1})\ | \ S_t=s\Big] \right|\le 3 \epsilon(s)
\end{align*}
%
\end{theorem}
\begin{proof}
First we expand $G_{t}^{\lambda} - \val(S_{t})$, from which we recover a series with the form of a return.
\begin{align}
G_{t}^{\lambda} - \val(S_{t})
&= R_{t+1} + \gamma_{t+1} ( 1- \lambda_{t+1} ) \val(S_{t+1}) - \val(S_{t}) 
+ \gamma_{t+1} \lambda_{t+1} (G_{t+1}^{\lambda} - \val(S_{t+1}))
\nonumber\\
&= R_{t+1} + \gamma_{t+1} \val(S_{t+1}) - \val(S_{t}) + \gamma_{t+1} \lambda_{t+1} (G_{t+1}^{\lambda} - \val(S_{t+1})) \label{eq:delta-expansion}
\end{align}

The variance of $G_{t}^{\lambda}$ is therefore 
\begin{align}
	\varg(s) &=
    \Exp{\left(G_{t}^{\lambda}-\Exp{G_{t}^{\lambda}|S_t=s}\right)^2|S_t=s} \nonumber\\
    &= \Exp{(G_{t}^{\lambda} - \val(s))^2|S_t=s} \label{eq_recursive_var}\\
    &= \Exp{
    	\Big(\delta_{t} + \gamma_{t+1}\lambda_{t+1}(G_{t+1}^{\lambda} - \val(S_{t+1}))\Big)^2
    |S_t=s}
   \nonumber\\
    &= \Exp{\delta_{t}^{2} | S_t = s}
    	+ \Exp{ \gamma_{t+1}^{2} \lambda_{t+1}^{2} (G_{t+1}^{\lambda} - \val(S_{t+1}))^2 | S_t=s}
        + 2 \Exp{ \gamma_{t+1} \lambda_{t+1} \delta_{t} (G_{t+1}^{\lambda} - \val(S_{t+1})) |S_t = s}\nonumber
\end{align}

Equation \eqref{eq_bellman_var} follows from Lemma \ref{delta_var} in the appendix, showing
$\Exp{ \gamma_{t+1} \lambda_{t+1} \delta_{t} (G_{t+1}^{\lambda} - \val(S_{t+1})) |S_t = s} = 0$.
\\
Now consider the case where we estimate the variance of the $\lambda$-return $\varghat$ 
of an approximate value function $\valhat$.
\begin{align*}
	\varghat(s)
    &= \Exp{(G_{t}^{\lambda} - \val(s) + \valhat(s) - \valhat(s) )^2|S_t=s}\\
    &= \Exp{(G_{t}^{\lambda} - \valhat(s))^2 | S_t=s} + (\valhat(s) - \val(s))^2
    +2\Exp{G_{t}^{\lambda} - \valhat(s)|S_t=s} (\valhat(s) - \val(s))
    .
    \end{align*}
This last term simplifies to
\begin{align*}
\!\!\!\Exp{G_{t}^{\lambda} \!-\! \valhat(s)|S_t\!=\!s} 
\!&= \! \Exp{G_{t}^{\lambda} \!-\! \val(s)|S_t\!=\!s} \!+\! \val(s) - \valhat(s)\\
&= \val(s) - \valhat(s)
\end{align*}
giving $ (\valhat(s) - \val(s))^2 + 2 (\val(s) - \valhat(s))(\valhat(s) - \val(s)) = -(\valhat(s) - \val(s))^2$.
We can use the same recursive form, therefore, as \eqref{eq_recursive_var}, giving
\begin{align*}
	\varghat(s) &= \Exp{\delta_t^2 + \gamma_{t+1}^{2} \lambda_{t+1}^{2}\varghat(S_{t+1}) | S_t = s}
        + 2 \Exp{ \gamma_{t+1} \lambda_{t+1} \delta_{t} (G_{t+1}^{\lambda} - \valhat(S_{t+1})) |S_t = s}
        - (\valhat(s) - \val(s))^2
\end{align*}
%
For the second term, 
\begin{align*}
\bigg|\Exp{ \gamma_{t+1} \lambda_{t+1} \delta_{t} (G_{t+1}^{\lambda} - \valhat(S_{t+1})) |S_t = s}\bigg|
= &\bigg|\Exp{ \gamma_{t+1} \lambda_{t+1} \delta_{t} (G_{t+1}^{\lambda} - \val(S_{t+1})) |S_t = s} \\
&+ \Exp{\gamma_{t+1} \lambda_{t+1} \delta_{t} (\val(S_{t+1}) - \valhat(S_{t+1})) |S_t = s}\bigg|\\
= &\bigg|\Exp{ \gamma_{t+1} \lambda_{t+1} \delta_{t} (\val(S_{t+1}) - \valhat(S_{t+1})) |S_t = s} \bigg|\\
\le &\epsilon(s)
.
\end{align*}
where the second equality follows from Lemma \ref{delta_var} and
the last step from the assumption about bounded covariance terms.
Therefore, 
\begin{align*}
\bigg|\varghat(s) &- \Exp{\delta_t^2 + \gamma_{t+1}^{2} \lambda_{t+1}^{2}\varghat(S_{t+1}) | S_t = s}\bigg|
\le 2 \epsilon(s) + (\valhat(s) - \val(s))^2 \le 3 \epsilon(s)
\end{align*}
\par \vspace{-0.8cm}
\end{proof}


\section{Experiments}

The primary purpose of these experiments is to demonstrate that both the \direct{} method and VTD can approximate the true expected \vartext{} under various conditions in the tabular setting. We consider two domains. The first is a deterministic chain, in Figure~\ref{fig:chain}, which is useful for basic evaluation and gives results which are easy to interpret. The second is a more complex MDP, in Figure~\ref{fig:complex_mdp}, with different discount and trace-decay parameters in each state. For all experiments Algorithm~\ref{eq:td_replacing_traces} is used as the value estimator. Unless otherwise stated, traces are not used ($\lambdaone=\lambdathree=0$). For each experimental setting 30 separate experiments were run and the estimates averaged, with standard deviation shown as shaded regions in the plots. The true values were determined by Monte Carlo estimation and are shown as dashed lines in the figures. Unless otherwise stated, the estimates are all initialized to zero.

We look at the effects of relative step-size between the value estimator and the variance estimators in Section~\ref{sec:exp:alpha}. In Section~\ref{sec:exp:state-dependent} we use the complex domain to show that both algorithms can estimate the variance with state-dependent $\gamma$ and $\lambdatwo$. 
In Section~\ref{sec:exp:err} we evaluate the two algorithms' responses to errors in the value estimate. Section~\ref{sec:exp:traces} looks at the effect of using traces in the estimation method. Finally, in Section~\ref{sec:exp:off-policy} we examine the off-policy setting.



\begin{figure}[!ht]
\centering \includegraphics[width=\linewidth]{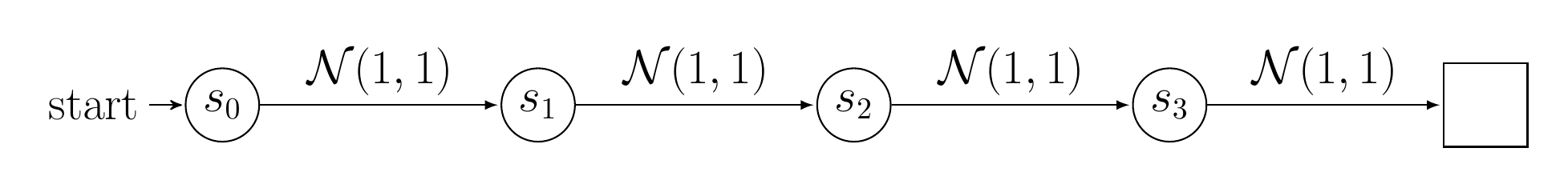}
\caption{Chain MDP with 4 non-terminal states and 1 terminal state. From each non-terminal state there is only a single action with a deterministic transition to the next state to the right. On each transition rewards are drawn from a normal distribution with mean and variance of 1.0. Evaluation was performed for $\lambdatwo=0.9$, which was chosen because it is not at either extreme and because 0.9 is a commonly used value for many RL experimental domains.}
\label{fig:chain}
\end{figure}



\begin{figure}[!ht]
	\begin{center}
	\centerline{\includegraphics[height=3in]{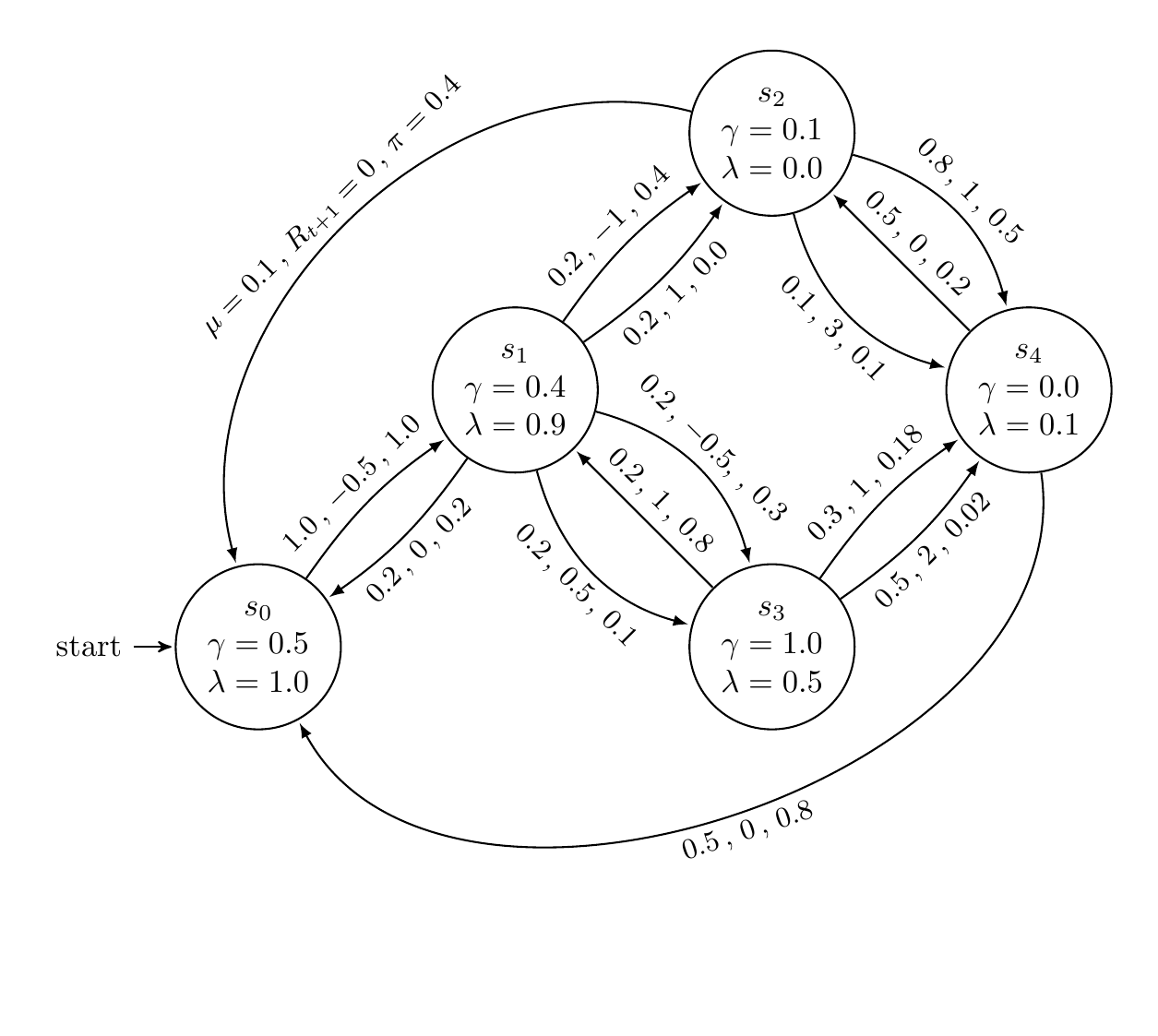}}
    \vspace{-0.2in}
	\caption{Complex MDP, with state-based values for $\gamma$ and $\lambdatwo$ and a stochastic policy. The state-dependent values of $\gamma$ and $\lambdatwo$ are chosen to provide a range of values, with at least one state acting as a terminal state where $\gamma=0$. On-policy action probabilities are indicated by $\mu$ and off-policy ones by $\pi$.}
	\label{fig:complex_mdp}
	\end{center}
\end{figure}



\subsection{The Effect of Step-size}
\label{sec:exp:alpha}

We use the chain MDP to investigate the impact of step-size choice. In Figure~\ref{fig:chain_l1.0_same} all step-sizes are the same $(\alpha=\bar{\alpha}=0.001)$. Both algorithms behave similarly. For Figure~\ref{fig:chain_l1.0_less} the step-size of the value estimate, $(\alpha=0.01)$, is greater than that of the variance estimators, $(\bar{\alpha}=0.001)$. The \direct{} algorithm smoothly approaches the correct value, while VTD first dips well below zero. This is to be expected as the estimates are initialized to zero and the variance is calculated as $\varghat(s)=\second(s)-\valhat(s)^2$. If the second moment lags behind the value estimate then the variance will be negative. In Figure~\ref{fig:chain_l1.0_more} the step-size for the variance estimators is larger than for the value estimator $(0.001=\alpha<\bar{\alpha}=0.01)$. While both methods overshoot the target, VTD has greater overshoot. For both cases of unequal step-size we see higher variance in the variance estimates for VTD.
%
%
\begin{figure}[p]
	\centering
    \subfigure[]{
    	\includegraphics[width=\linewidth]{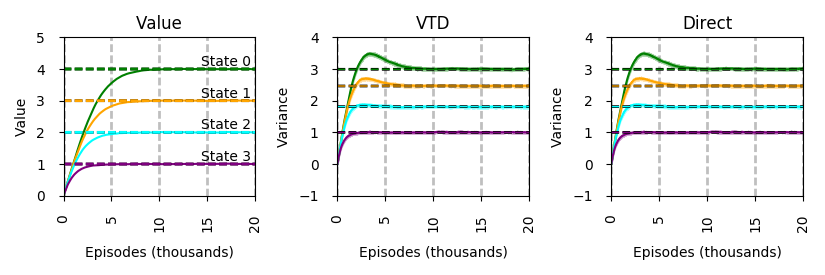}
		\label{fig:chain_l1.0_same}
     }
     \subfigure[]{
    	\includegraphics[width=\linewidth]{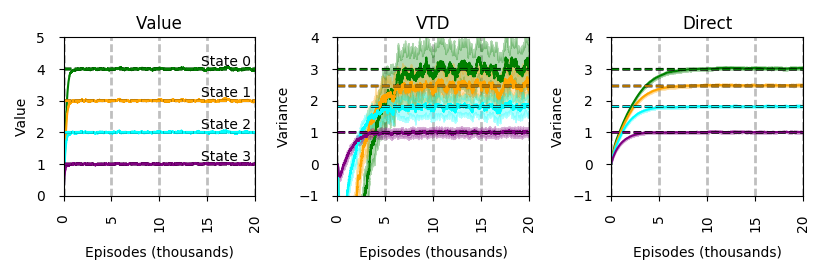}
        \label{fig:chain_l1.0_less}
     }
     \subfigure[]{
		\includegraphics[width=\linewidth]{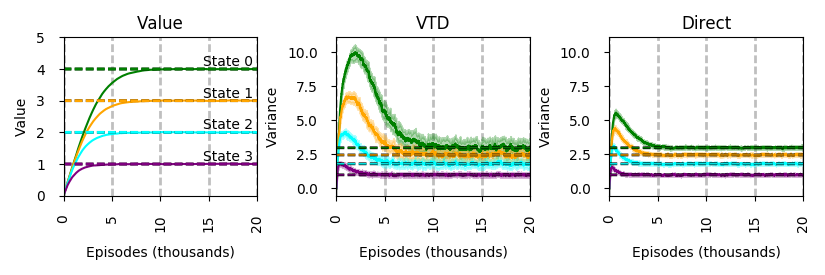}
		\label{fig:chain_l1.0_more}
     }
    \caption{\textbf{Chain MDP} ($\lambdatwo=0.9$). Varying the ratio of step-size between value and variance estimators. \textbf{a)} Step-sizes equal. $\alpha=\bar{\alpha}=0.001$. \textbf{b)} Variance step-size smaller. $\alpha=0.01, \bar{\alpha}=0.001$. \textbf{c)} Variance step-size larger. $\alpha=0.001, \bar{\alpha}=0.01$. We see greater variance in the estimates and greater over/undershoot for VTD when step-sizes are not equal.}
\end{figure}

Figure~\ref{fig:val_init_true_alpha_0} explores this further. Here the value estimator is initialized to the true values and updates are turned off ($\alpha=0$). The variance estimators are initialized to zero and learn with $\bar{\alpha}=0.001$, chosen simply to match the step-sizes used in the previous experiments. 
Despite being given the true values the VTD algorithm produces higher variance in its estimates, suggesting that VTD is dependent on the value estimator tracking. 

This sensitivity to step-size is shown in Figure~\ref{fig:alphas}. All estimates are initialized to their true values. For each ratio we computed the average variance of the 30 runs of 2000 episodes. We can see that the \direct{} method is largely insensitive to step-size ratio, but that VTD has higher mean squared error (MSE) except when the step-sizes are equal. This result holds for the other experimental settings of this paper, including the complex MDP, but further results are omitted for brevity.
\begin{figure}[!ht]
	\begin{center}
		\centerline{\includegraphics[width=\linewidth]{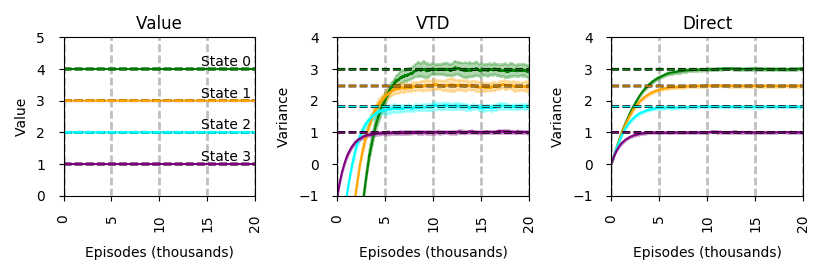}}
    	\caption{\textbf{Chain MDP} ($\lambdatwo=0.9$). Value estimate held fixed at the true values ($\alpha=0,\bar{\alpha}=0.001$). Notice the increased variance in the estimates for VTD, particularly in State 0.}
        	\label{fig:val_init_true_alpha_0}
	\end{center}
\end{figure}
\begin{figure}[!ht]
	\vskip -0.5em
	\begin{center}
		\centerline{\includegraphics[height=2.0in]{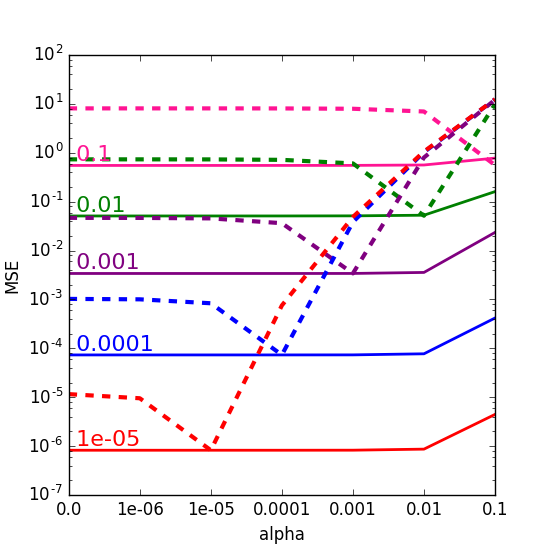}}
		\caption{\textbf{Chain MDP} ($\lambdatwo=0.9$). The MSE summed over all states as a function of ratios between the value step-size $\alpha$ (shown along the x-axis) and the variance step-size $\bar{\alpha}$ (shown as the 5 series). The \direct{} algorithm is indicated by the solid lines and VTD is indicated by the dashed. The MSE of the VTD algorithm is higher than the \direct{} algorithm, except when the step-size is the same for all estimators, $\alpha=\bar{\alpha}$ or for very small $\bar{\alpha}$.}
		\label{fig:alphas}
   \end{center}
   \vskip -2em
\end{figure}

These results beg the question, would there ever be a situation where different step-sizes between value and variance estimators is justified? Methods which automatically set the step-sizes may produce different values which are specific to the performance of each estimator. One such algorithm is ADADELTA, which adapts the step-size based on the TD error of the estimator \citep{Zeiler2012}. Figure~\ref{fig:adadelta} shows that using a separate ADADELTA step-size calculation for each estimator results in higher variance for VTD as expected (ADADELTA: $\rho=0.99, \epsilon=1e-6$), given that the value estimator and VTD produce different TD errors.
\begin{figure}[!ht]
\centering \includegraphics[width=\linewidth]{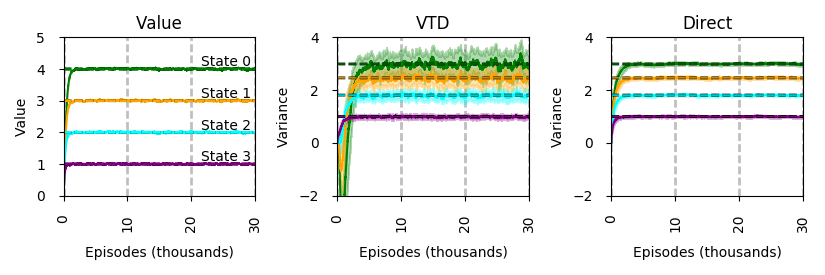}
\caption{\textbf{Chain MDP} ($\lambdatwo=0.9$). Results using ADADELTA algorithm to automatically and independently set the step-sizes $\alpha$ and $\bar{\alpha}$. The step-sizes produced are given in Appendix~\ref{sec:adadelta_step_size}.}
\label{fig:adadelta}
\end{figure}

\FloatBarrier
\subsection{Estimating for State-dependent \texorpdfstring{$\gamma$ and $\lambdatwo$}{Gamma and Lambda}.}
\label{sec:exp:state-dependent}

One of the contributions of VTD was the generalization to support state-based $\gamma$ and $\lambdatwo$. Here we evaluate the complex MDP from Figure~\ref{fig:complex_mdp} (in the on-policy setting, using $\mu$), which was designed for this scenario and which has a stochastic policy, is continuing, and has multiple possible actions from each state. Figure~\ref{fig:complex_mdp_onpolicy} shows that both algorithms estimate \vartext{} with similar results. This experiment was run with all step-sizes equal ($\alpha=\bar{\alpha}=0.01$).

\begin{figure}[!ht]
\centering \includegraphics[width=\linewidth]{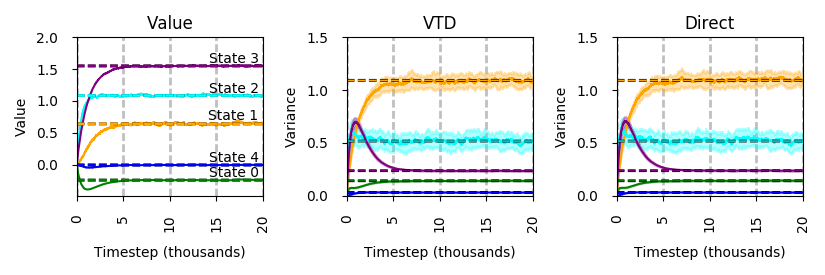}
\caption{\textbf{Complex MDP} evaluated on-policy with all step-sizes equal ($\alpha=\bar{\alpha}=0.01$).}
\label{fig:complex_mdp_onpolicy}
\end{figure}

\FloatBarrier
\subsection{Variable Error in the Value Estimates}
\label{sec:exp:err}
The derivation of our \direct{} algorithm assumes access to the true value function. The experiments of the previous sections demonstrate that both methods are robust under this assumption, in the sense that the value function was estimated from data and used to estimate \vartext{}. It remains unclear, however, how well these methods perform when the value estimates converge to biased solutions. 

To examine this we again use the complex MDP shown by Figure~\ref{fig:complex_mdp}. True values for the value functions and variance estimates are calculated from Monte Carlo simulation of 10,000,000 timesteps. For each run of the experiment each state of the value estimator was initialized to the true value plus an error ($\valhat(s)_0=\val(s)+\epsilon(s)$) drawn from a uniform distribution: $\epsilon(s)\in[-\zeta,\zeta]$, where $ \zeta=\max_s(|v(s)|)*\text{err ratio}$ 
(the maximum value in this domain is 1.55082409). The value estimate was held constant throughout the run $(\alpha=0.0)$. The experiment consisted of 120 runs of 80,000 timesteps. To look at the steady-state response of the algorithms we use only the last 10,000 timesteps in our calculations. Figure~\ref{fig:complex_error_rand_std_dev} plots the average variance estimate for each state. Additionally we show the average standard deviation of the estimates in the shaded regions. Sweeps over step-size were conducted, $\bar{\alpha}\in[0.05, 0.04, 0.03, 0.02, 0.01, 0.007, 0.005, 0.003, 0.001]$, and the MSE evaluated for each state. Each data point is for the step-size with the lowest MSE for that error ratio and state. While the average estimate is closer to the true values for VTD, the variance of the estimates is much larger. Further, the average estimates for VTD are either unchanged or move negative, while those of the \direct{} algorithm tend toward positive bias.
\begin{figure}[!ht]
\centering \includegraphics[height=2in]{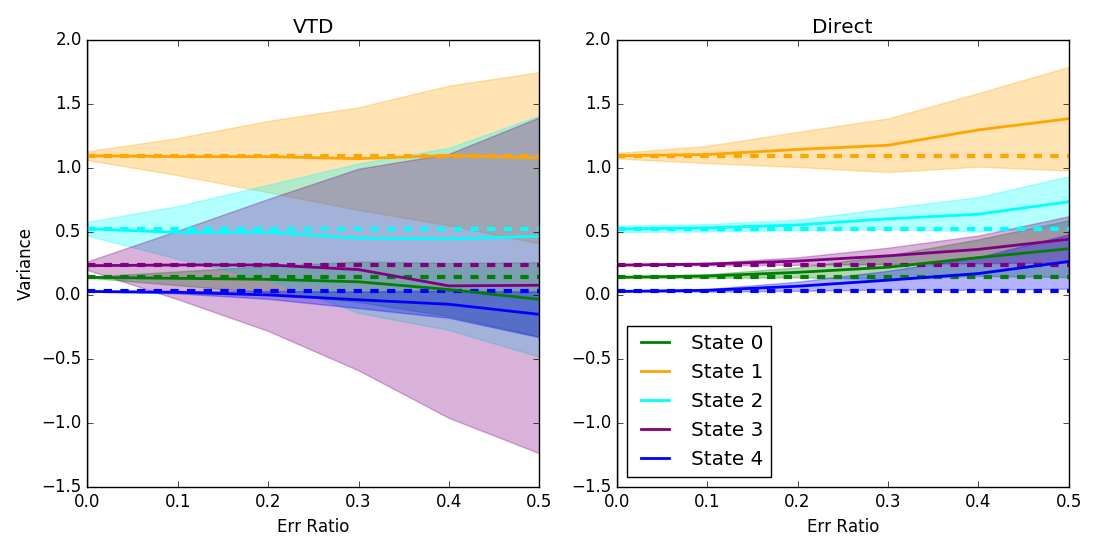}
\caption{\textbf{Complex MDP}. For each run the value estimate of each state is offset by a random amount from uniform distribution whose size is a function of the Err Ratio and the maximum true value in the MDP. Standard deviation of the estimates is shown by shading.}
\label{fig:complex_error_rand_std_dev}
\end{figure}

For Figure~\ref{fig:best_alpha} the MSE is summed over all states. Again, for each error ratio the MSE was compared over the same step-sizes as before and for each point the smallest MSE is plotted. 
\begin{figure}[!ht]
	\begin{center}    
	\centerline{\includegraphics[height=2in]{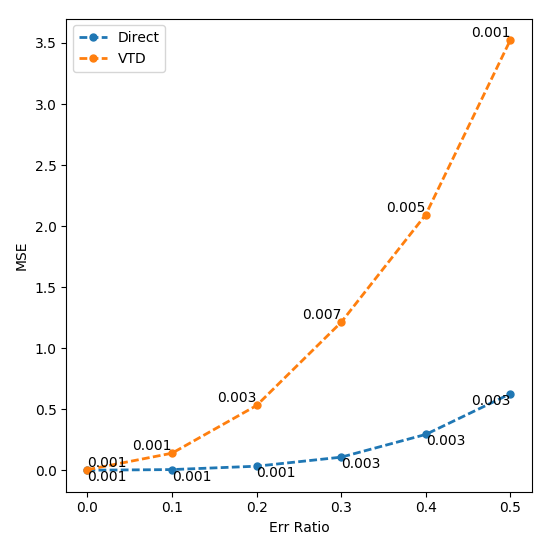}}
	\caption{\textbf{Complex MDP}. The MSE computed for the last 10,000 timesteps of 120 runs summed over all states for the step-size with the lowest overall MSE at each error ratio. For each point the step-size used ($\alpha=\bar{\alpha}$) is displayed.}
	\label{fig:best_alpha}
	\end{center}
\end{figure}
These results suggest the \direct{} algorithm is less affected by error in $\valhat$.

\FloatBarrier
\subsection{Experiments with Traces}
\label{sec:exp:traces}

In this section we briefly look at the behavior of the complex domain when traces are used. For Figure~\ref{fig:complex_traces_l0_lb_1} traces are used for the variance estimators, but not for policy evaluation ($\lambdaone=0.0,\lambdathree=1.0$) and the step-sizes are all equal (0.01). 
Here we see no significant difference between VTD and the \direct{} algorithm. 
For Figure~\ref{fig:complex_traces_l1_lb_0} we look at the opposite scenario, where traces are used for policy evaluation, but not in the variance estimators ($\lambdaone=1.0, \lambdathree=0.0$). Here we do see a difference, particularly the VTD method shows more variance in its estimates for State 0 and 3.

\begin{figure}[!ht]
	\centering
    \subfigure[]{
		\includegraphics[width=\columnwidth] {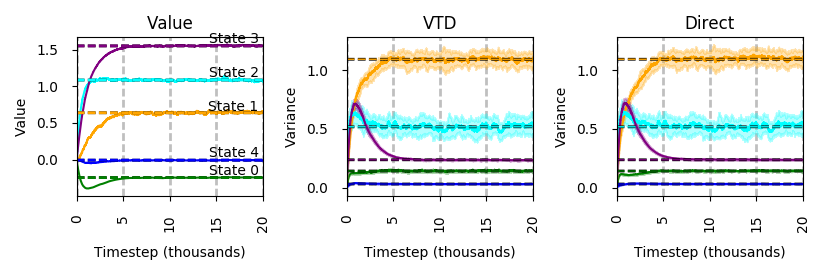}
		\label{fig:complex_traces_l0_lb_1}
	}
    \subfigure[]{
		\includegraphics[width=\columnwidth]{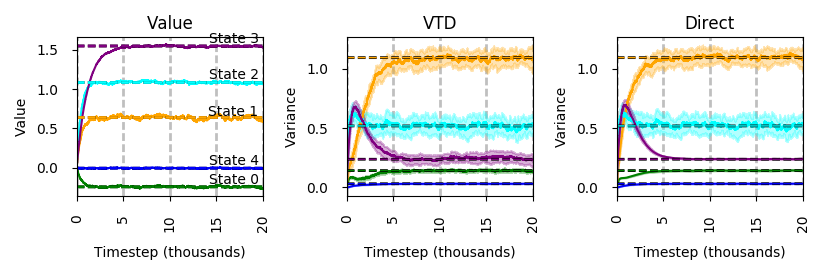}
		\label{fig:complex_traces_l1_lb_0}
    }
    \caption{\textbf{Complex MDP}. Using traces (TD($\lambda$), $\alpha=\bar{\alpha}=0.01$). \textbf{a)} Traces used in variance estimators ($\lambdaone=0.0,\lambdathree=1.0$), \textbf{b)} Traces used in value estimator ($\lambdaone=1.0,\lambdathree=0.0$). Notice the slight increase in the variance of the VTD estimates for State 0 and 3.}
\end{figure}
%
%

\FloatBarrier
\subsection{Experiments in an Off-policy Setting}
\label{sec:exp:off-policy}

In the off-policy setting the agent follows a behavior policy $\mu$, but is estimating the value of a target policy $\pi$. The ratio between these two policies is called the importance sampling ratio, $\rho=\frac{\pi(s,a)}{\mu(s,a)}$, and is used to modify the value function update. 

We evaluate two different off-policy scenarios on the complex MDP. In the first scenario we estimate \vartext{} under the target policy from off-policy samples. That is, we estimate \vartext{} that would be observed if we were following the target policy. In this scenario $\secondrho=1,\bar{\rho}=\rho$. Figure~\ref{fig:complex_offpolicy_same} shows that both methods are able achieve the same results in this setting.

\begin{figure}[t]
\begin{center}
	\centerline{\includegraphics[width=\columnwidth]{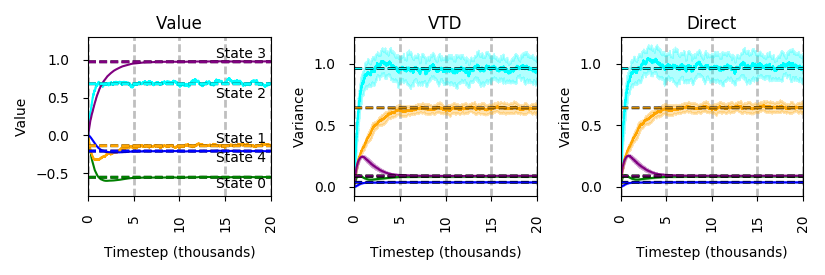}}
  \caption{\textbf{Complex MDP} estimating \vartext{} from off-policy samples ($\alpha=\bar{\alpha}=0.01, \secondrho=1,\bar{\rho}=\rho$).}
  \label{fig:complex_offpolicy_same}
\end{center}
\end{figure}

In the second off-policy setting we estimate the variance of the off-policy return, which is the return being used to update the value estimator and is simply the multiplication of the $\lambda$-return by $\rho$. In this scenario $\bar{\rho}=1$ and $\secondrho=\rho$. Figure~\ref{fig:complex_offpolicy_2_same} shows that both algorithms successfully estimate the return in this setting. However, despite having the same step-size as the value estimator, VTD produces higher variance in its estimates, as is most clearly seen in State 3.
\begin{figure}[t]
	\begin{center}
	\centerline{\includegraphics[width=\linewidth]{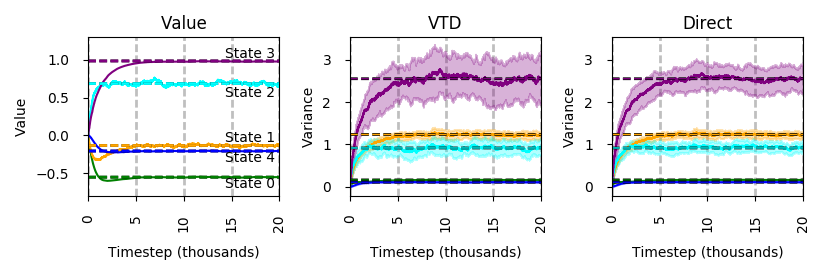}}
\caption{\textbf{Complex MDP} estimating the variance of the off-policy return ($\alpha=\bar{\alpha}=0.01,\bar{\rho}=1,\secondrho=\rho$).}
\label{fig:complex_offpolicy_2_same}
	\end{center}
\end{figure}

\section{Discussion}

Both the \direct{} method and VTD effectively estimate the variance across a range of settings, but the \direct{} method is simpler and more robust. This simplicity alone makes the \direct{} method preferable. The higher variance in estimates produced by VTD is likely due to the inherently larger target which VTD uses in its learning updates: $\expect [X^2] \geq \expect [(X - \expect [X])^2]$; we show more explicitly how this affects the updates of VTD in Appendix~\ref{sec:updates}. One would expect the differences between the two approaches to be most pronounced for domains with larger returns than those demonstrated here. Our focus was simple MDPs. In such settings we can define clear experiments where the properties of these variance estimation algorithms can be carefully evaluated isolated from additional effects like state-aliasing due to function approximation. Consider the task of helicopter hovering formalized as a reinforcement learning task \citep{Ng2004}. In the most well-known variants of this problem the agent receives massive negative reward for crashing the helicopter (e.g., minus one million). In such problems the magnitude and variance of the return is large. In such cases, estimating the second moment may not be feasible from a statistical point of view, whereas the target of our direct variance estimate should be better behaved. 

We focused on the tabular case, where each state is represented uniquely. Future work will investigate extending our theoretical characterization and experiments to the function approximation case. Our algorithm extends naturally with little modification. To extend the theory, there have been some promising results characterizing fixed points under the projected Bellman operator for the second moment \citep{Tamar2016}. An extension to projected Bellman operators could also further help bound errors incurred from inaccuracies in the value function. 


\section{Conclusion}
In this paper we introduced a simple method for estimating the variance of the $\lambda$-return using temporal difference learning.
Our approach is simpler than existing approaches, and appears to work better in practice. We performed an extensive empirical study. Our findings suggest that our new method outperforms VTD when: (1) there is a mismatch in step-size between the value estimator and the variance estimator, (2) traces are used with the value estimator, (3) estimating variances of the off-policy return, and (4) there is error in the value estimate.


\section*{Acknowledgements}
Funding for this work was provided by the Natural Sciences and Engineering Research Council of Canada, Alberta Innovates, and Google DeepMind.

\clearpage
\subsubsection*{References}
\printbibliography

\clearpage

\appendix

\section{Variance Estimation in the Off-Policy Setting}
\label{appendix:general}

Value estimates are made with respect to a target policy, $\pi$. If the behavior policy, $\mu$, is the same as the target policy then we say that samples are collected on-policy and when they are not the same, the samples are collected off-policy. A common approach for off-policy learning algorithms is to weight each update by the importance sampling ratio: $\rho_t=\frac{\pi(S_t,A_t)}{\mu(S_t,A_t)}$. Off-policy estimates are then implemented by multiplying the trace updates by $\rho_t$:
\begin{align*}
\trace_t(s)&\leftarrow\begin{cases}
      \rho_t(\gamma_t\lambda_t \trace_{t-1}(s) + 1) & s=S_t\\
      \rho_t\gamma_t\lambda_t \trace_{t-1}(s) & \forall s \in \states, s \ne S_t
   \end{cases}.
\end{align*}

There are two different scenarios to be considered in the off-policy setting. The first scenario is estimating the variance of the (on-policy) $\lambda$-return of the target policy, while following a different behavior policy. 
%
In the second setting, the goal is to estimate the variance of the off-policy $\lambda$-return.  
The off-policy $\lambda$-return is
\begin{equation*}
\!\!G_t^{\lambda}
\!=\!\rho_{t} \Big(
\C_{t+1}+\gamma_{t+1}(1\!-\!\lambda_{t+1})\val_t(S_{t+1}) +\gamma_{t+1}\lambda_{t+1} G^{\lambda}_{t+1} \Big).
\end{equation*}
where the multiplication by the potentially large importance sampling ratios can significantly increase variance. 


It is important to note that you would only ever estimate one or the other off-policy variance with a given estimator. Let $\secondrho$ be the weighting for the value estimator, and $\bar{\rho}$ the weighting for the variance estimator. If estimating the variance of the target return from off-policy samples, the first scenario, $\secondrho_t=1\ \forall t$ and $\bar{\rho}_t=\rho_t$. If estimating the variance of the off-policy return $\bar{\rho}_t=1\ \forall t$ and $\secondrho_t=\rho_t$. 

Here we present the resulting algorithms which use TD($\lambda$) estimators with accumulating traces.

\textbf{Direct Variance Algorithm}
\begin{equation}
\begin{aligned}
\bar{\C}_{t+1}& \leftarrow (\secondrho_t\delta_t + (\secondrho_t-1)\valhat_{t+1}(s))^2\\
\bar{\gamma}_{t+1}& \leftarrow \gamma_{t+1}^2 \lambdatwo_{t+1}^2\secondrho_t^2\\
\bar{\delta}_t& \leftarrow \bar{\C}_{t+1} + \bar{\gamma}_{t+1}\varghat_t(s') - \varghat_t(s)\\
\bar{\trace}_t(s)&\leftarrow\begin{cases}
     \bar{\rho}_t(\bar{\gamma}_t\lambdathree_t \bar{\trace}_{t-1}(s) + 1) & s=S_t \\
      \bar{\rho}_t(\bar{\gamma}_t\lambdathree_t \bar{\trace}_{t-1}(s)) & \forall s \in \states, s \ne S_t
   \end{cases}\\
\varghat_{t+1}(s)&\leftarrow \varghat_t(s)+\bar{\alpha} \bar{\delta}_t \bar{\trace}_t(s)
\end{aligned}
\label{alg:direct_off_policy}
\end{equation}

Variance is computed directly as $\varghat_{t+1}(s)$.

\textbf{Second Moment Algorithm}
\begin{equation}
\begin{aligned}
\bar{G}_t&\leftarrow \C_{t+1} + \gamma_{t+1}(1-\lambdatwo_{t+1})\valhat_{t+1}(s')\\
\bar{\C}_{t+1}& \leftarrow \secondrho_t^2 \bar{G}_t^2 + 2\secondrho_t^2 \gamma_{t+1} \lambdatwo_{t+1} \bar{G}_t \valhat_{t+1}(s')\\
\bar{\gamma}_{t+1}& \leftarrow \secondrho_t^2 \gamma_{t+1}^2 \lambdatwo_{t+1}^2\\
\bar{\delta}_t& \leftarrow \bar{\C}_{t+1} + \bar{\gamma}_{t+1}\second_t(s') - \second_t(s)\\
\bar{\trace}_t(s)&\leftarrow\begin{cases}
     \bar{\rho}_t(\bar{\gamma}_t\lambdathree_t \bar{\trace}_{t-1}(s) + 1) & s=S_t \\
     \bar{\rho}_t(\bar{\gamma}_t\lambdathree_t \bar{\trace}_{t-1}(s)) & \forall s \in \states, s \ne S_t
   \end{cases}\\
\second_{t+1}(s)&\leftarrow \second_t(s)+\bar{\alpha} \bar{\delta}_t\bar{\trace}_t(s)
\end{aligned}
\label{alg:comparison_off_policy}
\end{equation}

Variance is computed as $\varghat_{t+1}(s)=M_{t+1}(s)-\valhat_{t+1}(s)^2$. 

For convenience we summarize the variables used:
\begin{align*}
\valhat-&\text{estimated value function of the target policy $\pi$.}\\
\C-&\text{reward used in the value function estimate.}\\
\bar{\C}-&\text{\metareward{} used in the variance estimate.}\\
\lambdatwo -&\text{bias-variance parameter of the target $\lambda$-return.}\\
\lambdaone -&\text{trace-decay parameter of the value estimator.}\\
\lambdathree -&\text{trace-decay parameter of the secondary estimator.}\\
\gamma -&\text{discounting function used by the \valtext{} estimator.}\\
\bar{\gamma} -&\text{discounting function used by the \vartext{} estimator.}\\
\delta_t -&\text{TD error of the value function at time $t$.}\\
\bar{\delta_t} -&\text{TD error of the variance estimator at time $t$.}\\
\second-&\text{estimate of the second moment.}\\
\varghat-&\text{estimate of the variance.}\\
\bar{\rho}-&\text{importance sampling ratio for estimating the }\\
&\text{variance of the target return from off-policy samples.}\\
\secondrho-&\text{importance sampling ratio used to estimate the} \\
&\text{variance of the off-policy return.}
\end{align*}

\section{Bellman Operators for the Variance in the Off-Policy Setting}

\begin{lemma}
\label{delta_var}
For $\val(s) = \Exp{ G_{t+1}^{\lambda}|S_t=s}$, i.e., satisfying the Bellman equation, for any bounded
function $b: \states \times \actions \times \mathbb{R} \times \states \rightarrow \mathbb{R}$, 
\begin{equation*}
\expect[b(S_t, A_t, R_{t+1}, S_{t+1})(G_{t+1}^\lambda-\val(S_{t+1}))|S_t=s]=0\ \end{equation*}
\end{lemma}
\begin{proof}
Let $b_t = b(S_t, A_t, R_{t+1}, S_{t+1})$.
By the law of total expectation:
\begin{align*}
  \Exp{ b_t (G_{t+1}^{\lambda}- \val(S_{t+1})) |S_t=s}
  = \expect\left[
    \expect[b_{t} (G_{t+1}^{\lambda} \!- \val(S_{t+1})) |S_t, A_{t}, S_{t+1}]|S_t=s\right]
\end{align*}

Given $S_{t}$, $A_t$, $R_{t+1}$ and $S_{t+1}$,
$b_{t}$ is constant and can be moved outside of the expectation. Therefore,
\begin{align*}
\expect&[b_{t} (G_{t+1}^{\lambda} - \val(S_{t+1})) \Big|S_{t}, A_{t}, R_{t+1}, S_{t+1}]
=\Exp{ b_{t} \big|S_{t}, A_{t}, R_{t+1}, S_{t+1} }
\times \Exp{ G_{t+1}^{\lambda} - \val(S_{t+1}) \big|S_{t}, A_{t}, R_{t+1}, S_{t+1}}
\end{align*}
Because
\begin{equation*}
\Exp{ G_{t+1}^{\lambda} - \val(S_{t+1}) \big|S_{t}, A_{t}, R_{t+1}, S_{t+1}}=0
\end{equation*}
the result follows.
\end{proof} 


\begin{theorem}
\label{recurse_var_off_policy}
\begin{align*}
\varg(s)
=\expect[(\secondrho_t\delta_t +(\secondrho_t-1)\val(s))^2+\lambda_{t+1}^2\gamma_{t+1}^2\secondrho_t^2\varg(S_{t+1})|S_t=s]
\end{align*}
\end{theorem}
\begin{proof}
The proof is similar to the proof of Theorem \ref{recurse_var}.
%
\begin{align*}
\varg(s)
=&\expect[\{G_t^\lambda - \val(S_t)\}^2|S_t=s]\\
=& \expect[\{\secondrho_t\C_{t+1}+\secondrho_t\gamma_{t+1}(1-\lambda_{t+1})\val(S_{t+1}) 
+ \secondrho_t\gamma_{t+1}\lambda_{t+1} G_{t+1}^\lambda - v(s)\}^2|S_t=s]\\
=& \expect[\{\secondrho_t\C_{t+1}+\secondrho_t\gamma_{t+1}\val(S_{t+1})-\secondrho_t \val(s) + \secondrho_t \val(s) 
- \secondrho_t\gamma_{t+1}\lambda_{t+1} \val(S_{t+1})\\
\ &+ \secondrho_t\gamma_{t+1}\lambda_{t+1} G_{t+1}^\lambda-\val(s)\}^2|S_t=s]\\
=& \expect[\{(\secondrho_t\delta_t\ + (\secondrho_t - 1)\val(s)) + \secondrho_t\gamma_{t+1}\lambda_{t+1}(G_{t+1}^\lambda-\val(S_{t+1}))\}^2|S_t=s]\\
=& \expect[(\secondrho_t\delta_t + (\secondrho_t -1)\val(s))^2
+\secondrho_t^2\gamma_{t+1}^2\lambda_{t+1}^2(G_{t+1}^\lambda- \val(S_{t+1}))^2\\
\ &+2\secondrho_t \gamma_{t+1}\lambda_{t+1} (\secondrho_t\delta_t + (\secondrho_t - 1)\val(s))(G_{t+1}^\lambda- \val(S_{t+1}))|S_t=s]\\
=& \expect[(\secondrho_t\delta_t + (\secondrho_t - 1)\val(s))^2
+\secondrho_t^2\gamma_{t+1}^2\lambda_{t+1}^2(G_{t+1}^\lambda- \val(S_{t+1}))^2\\
\ &+2\secondrho_t^2 \gamma_{t+1}\lambda_{t+1} \delta_t (G_{t+1}^\lambda- \val(S_{t+1}))
+ 2\secondrho_t \gamma_{t+1}\lambda_{t+1}(\secondrho_t-1)\val(s)(G_{t+1}^\lambda-\val(S_{t+1}))|S_t=s]
\end{align*}
Using Lemma \ref{delta_var}, with different fixed functions $b$, we can conclude that the last two terms are zero, giving
\begin{align*}
\varg(s) =&\expect[(\secondrho_t\delta_t + (\secondrho_t-1)\val(s))^2
+\secondrho_t^2\gamma_{t+1}^2\lambda_{t+1}^2(G_{t+1}^\lambda-\val(S_{t+1}))^2|S_t=s]\\
\intertext{By the law of total expectation} 
\varg(s)
=&\expect[(\secondrho_t\delta_t + (\secondrho_t-1)\val(s))^2
+\expect[\secondrho_t^2\gamma_{t+1}^2\lambda_{t+1}^2(G_{t+1}^\lambda-\val(s'))^2|S_{t+1}=s']|S_t=s]\\
=&\expect[(\secondrho_t\delta_t+(\secondrho_t-1)\val(s))^2
+ \secondrho_t^2\gamma_{t+1}^2\lambda_{t+1}^2\varg(S_{t+1})|S_t=s].
\end{align*}
completing the proof.
\end{proof}

Theorem \ref{recurse_var_off_policy} gives a Bellman equation for $\varghat(s)$ in the more general off-policy setting. 
The resulting TD algorithm uses \metareward{} $(\secondrho_t\delta_t +(\secondrho_t-1)\val(s))^2$ and discounting function $\secondrho_t^2\gamma_{t+1}^2\lambda^2$.

\section{ADADELTA Step-Sizes}
\label{sec:adadelta_step_size}

The step-sizes generated by the ADADELTA algorithm in Figure~\ref{fig:adadelta} are shown in Figure~\ref{fig:adadelta_stepsizes}. As we evaluate in the tabular case at each timestep only the step-size for the current state has any impact. Thus, the values shown here are the average step-size used over each episode.

\begin{figure}[!ht]
\centering \includegraphics[height=2in]{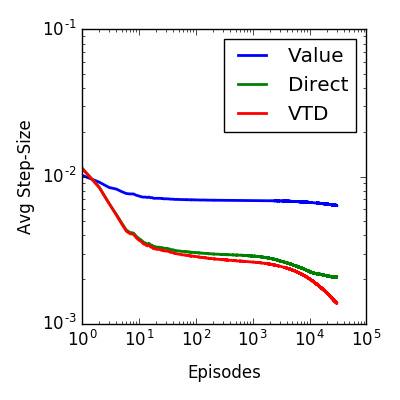}
\caption{\textbf{Chain MDP}. The average step-sizes computed for ADADELTA in Figure~\ref{fig:adadelta}.}
\label{fig:adadelta_stepsizes}
\end{figure}


\section{Variability in Updates}
\label{sec:updates}
In this section, we show the effective update to $\varghat_t(s)$ on each timestep for each of the two algorithms in the on-policy setting. For notational clarity let $\c=\c_{t+1},\alpha=\alpha_t, \gamma=\gamma_{t+1}, \lambdatwo=\lambdatwo_{t+1},s=s_{t},s'=s_{t+1},\delta_t=\delta$.

For the direct algorithm the change is just:

\begin{equation}
\begin{aligned}
\Delta \varghat_t(s)=\bar{\alpha}(\delta^2 + \bar{\gamma}\varghat_t(s')-\varghat_t(s)).
\label{eq:direct_update}
\end{aligned}
\end{equation}

The updates for the VTD algorithm are much more complicated to compute and we will make some assumptions about the domain in order to simplify the derivation. First we compute the change in the second moment and value estimators separately.

We first expand the term $\delta^2$:
\begin{align*}
\delta=&\c+\gamma \valhat_t(s')-\valhat_t(s)\\
\delta^2=&(\c+\gamma \valhat_t(s'))^2 
- 2(\c+\gamma \valhat_t(s'))\valhat_t(s)+ \valhat_t(s)^2.
\end{align*}

Now we expand the change in the second moment estimate, $M$. To simplify the expansion we make the assumption that at each transition the agent moves to a new state, i.e. $s_t\ne s_{t+1}\forall t$ (this is not required for our algorithm, but simplifies the expansions below). This assumption holds for both of the domains examined in this paper. This allows us to substitute $\valhat_{t+1}(s')=\valhat_{t}(s)$, which greatly simplifies the updates.

\begin{align*}
\Delta M(s)=&\bar{\alpha}[(\c+\gamma \valhat_{t+1}(s'))^2
- \bar{\gamma}^2 \valhat_{t+1}(s')^2 + \bar{\gamma}M_t(s')-M_t(s)]\\
=&\bar{\alpha}[(\c+\gamma \valhat_{t}(s'))^2 
- \bar{\gamma}^2 \valhat_{t}(s')^2 + \bar{\gamma}M_t(s')-M_t(s)]\\
=&\bar{\alpha}[(\c+\gamma \valhat_{t}(s'))^2 -2(\c+\gamma \valhat_t(s'))\valhat_t(s)
+ \valhat_t(s)^2 + 2(\c+\gamma \valhat_t(s'))\valhat_t(s) - \valhat_t(s)^2 \\
&- \bar{\gamma}^2 \valhat_{t}(s')^2 + \bar{\gamma}M_t(s')-M_t(s)]\\
=&\bar{\alpha}[\delta^2 + 2(\c+\gamma \valhat_t(s'))\valhat_t(s) - \valhat_t(s)^2 
- \bar{\gamma}^2 \valhat_{t}(s')^2 + \bar{\gamma}M_t(s')-M_t(s)]\\
\intertext{Notice that from the definition of the TD error: $R+\gamma \valhat_t(s')=\delta+\valhat_t(s)$.}
=&\bar{\alpha}[\delta^2 + 2(\delta+\valhat_t(s))\valhat_t(s) - \valhat_t(s)^2 
- \bar{\gamma}^2 \valhat_{t}(s')^2 + \bar{\gamma}M_t(s')-M_t(s)]\\
=&\bar{\alpha}[\delta^2 + 2\delta \valhat_t(s) + \valhat_t(s)^2 - \bar{\gamma}^2 \valhat_{t}(s')^2 
+ \bar{\gamma}M_t(s')-M_t(s)]\\
=&\bar{\alpha}[\delta^2 + (\bar{\gamma}M_t(s')-\bar{\gamma}\valhat_t(s')^2) 
-(M_t(s)-\valhat_{t}(s)^2) + 2\delta \valhat_t(s) - \bar{\gamma}^2 \valhat_{t}(s')^2
+\bar{\gamma}\valhat_t(s')^2]\\
=&\bar{\alpha}[\delta^2 + \bar{\gamma}\varghat_t(s') - \varghat_t(s) + 2\delta \valhat_t(s)
- \bar{\gamma}^2 \valhat_{t}(s')^2 + \bar{\gamma}\valhat_t(s')^2]\\
=&\bar{\alpha}[\delta^2 + \bar{\gamma}\varghat_t(s') - \varghat_t(s)]
+ \bar{\alpha}[2\delta \valhat_t(s) - \bar{\gamma}^2 \valhat_{t}(s')^2 + \bar{\gamma}\valhat_t(s')^2]\\
\end{align*}

The first half of this equation is the same as the update for the direct algorithm \eqref{eq:direct_update}. Now we expand the change in the variance update for VTD:

\begin{align*}
\Delta \varghat_t(s) =&(M_{t+1}(s) - \valhat_{t+1}(s)^2) - (M_{t}(s) - \valhat_{t}(s)^2)\\
=&\Delta M(s) + \valhat_{t}(s)^2 - \valhat_{t+1}(s)^2\\
=&\Delta M(s) + \valhat_{t}(s)^2 - (\alpha\delta + \valhat_{t}(s))^2\\
=&\Delta M(s) + \valhat_t(s)^2 
- ((\alpha\delta)^2 + 2\alpha\delta \valhat_t(s) + \valhat_t(s)^2)\\
=&\Delta M(s) -(\alpha\delta)^2 - 2\alpha\delta \valhat_t(s).
\end{align*}

Note that in the case that $\alpha=\bar{\alpha}$ this last term cancels out and we're left with:

\begin{align*}
\Delta \varghat_t(s)=&\alpha[\delta^2 + \bar{\gamma}\varghat_t(s') - \varghat_t(s)] +
\alpha \valhat_t(s')^2(\bar{\gamma}- \bar{\gamma}^2) -(\alpha\delta)^2.
\end{align*}

This suggests that VTD will deviate from the direct method more when: $\alpha$ is larger, $\valhat_t(s')$ is larger, $\bar{\gamma}=0.5$ and for large values of $\delta$. In general, we expect from this equation that the updates for the VTD will be larger than those of the direct method, suggesting a cause for the higher variance of variance estimates across runs as observed for VTD under a number of scenarios.

We also empirically tested this hypothesis, with Table~\ref{table:updates} showing the updates for the two algorithms across the various experiments. For episodic tasks (chain MDP, Figures~\ref{fig:chain_l1.0_same}-\ref{fig:adadelta}) the results show the average total absolute change over all states for a given episode averaged across runs and then averaged across all episodes. For the continuing case (complex MDP, Figures~\ref{fig:complex_mdp_onpolicy}-\ref{fig:complex_offpolicy_2_same}) the results are the average absolute change for a timestep averaged over all runs and then averaged across the entire run length. The experiments shaded in gray are those where the two algorithms behaved nearly identically. In this case we see that the average magnitude of updates is nearly identical. For the other experiments the VTD algorithm showed higher variance in its variance estimates across runs. For these experiments we see that the average magnitude of the VTD updates is much larger than for the direct algorithm.

\definecolor{gray}{rgb}{0.8, 0.8, 0.8}

\begin{table}[t]
\caption{Average updates for various experiments.}
\label{table:updates}
\begin{center}
\begin{tabular}{|l|c|c|c|c|}
\hline
Fig. & Value & Snd Mmnt & VTD & Direct\\
\hline
\rowcolor{gray}\ref{fig:chain_l1.0_same} & 0.00332 & 0.0157 & 0.00415 & 0.00415\\
\hline
\ref{fig:chain_l1.0_less} & 0.0322 & 0.0165 & 0.143 & 0.00387\\
\hline
\ref{fig:chain_l1.0_more} & 0.00332 & 0.156 & 0.142 & 0.0419\\
\hline
\ref{fig:val_init_true_alpha_0} & 0.0 & 0.0166 & 0.0166 & 0.00381\\
\hline
\ref{fig:adadelta} & 0.0212 & 0.0306 & 0.0752 & 0.00884\\
\hline
\rowcolor{gray}\ref{fig:complex_mdp_onpolicy} & 0.00362 & 0.00675 & 0.00381 & 0.00385\\
\hline
\rowcolor{gray}\ref{fig:complex_offpolicy_same} & 0.00362 & 0.00461 & 0.00303 & 0.00307\\
\hline
\ref{fig:complex_offpolicy_2_same} & 0.00362 & 0.0110 & 0.0116 & 0.00838\\
\hline
\end{tabular}
\end{center}
\end{table}

\end{document}